\documentclass{article}
\pdfoutput=1
\usepackage{arxiv}
\usepackage{lineno,hyperref}

\usepackage{amssymb}
\usepackage{amsthm}
\usepackage{amsmath}

\newtheorem{theorem}{Theorem}

\newtheorem{definition}{Definition}[]

\usepackage{hyperref}

\usepackage{algorithm}
\usepackage{algpseudocode}

\usepackage{tikz}
\usetikzlibrary{plotmarks}
\usepackage{pgfplots}
\usepackage{svg}
\usetikzlibrary{positioning,shapes,arrows,shadows,patterns}

\usepackage{multirow}
\usepackage{colortbl}

\usepackage{amsmath,amsfonts}
\usepackage{array}
\usepackage[caption=false,font=normalsize,labelfont=sf,textfont=sf]{subfig}
\usepackage{textcomp}
\usepackage{stfloats}
\usepackage{url}
\usepackage{verbatim}
\usepackage{graphicx}

\title{Superior Scoring Rules for Probabilistic Evaluation of Single-Label Multi-Class Classification Tasks}

\author{Rouhollah Ahmadian \\
  Department of Mathematics and Computer Science \\
  Amirkabir University of Technology (Tehran Polytechnic)\\
  Iran\\
  \texttt{rahmadian@aut.ac.ir}\\
   \And
  Mehdi Ghatee\footnote{The corresponding author} \\
  Department of Mathematics and Computer Science \\
  Amirkabir University of Technology (Tehran Polytechnic)\\
  Iran\\
  \texttt{ghatee@aut.ac.ir}\\
 \And
  Johan Wahlstr\"om \\
  Department of Computer Science\\
 University of Exeter\\
UK\\
  \texttt{j.wahlstrom@exeter.ac.uk}\\
}

\begin{document}
\maketitle
\begin{abstract}
This study introduces novel superior scoring rules called Penalized Brier Score (\textit{PBS}) and Penalized Logarithmic Loss (\textit{PLL}) to improve model evaluation for probabilistic classification. Traditional scoring rules like Brier Score and Logarithmic Loss sometimes assign better scores to misclassifications in comparison with correct classifications. This discrepancy from the actual preference for rewarding correct classifications can lead to suboptimal model selection. By integrating penalties for misclassifications, \textit{PBS} and \textit{PLL} modify traditional proper scoring rules to consistently assign better scores to correct predictions. Formal proofs demonstrate that \textit{PBS} and \textit{PLL} satisfy strictly proper scoring rule properties while also preferentially rewarding accurate classifications. Experiments showcase the benefits of using \textit{PBS} and \textit{PLL} for model selection, model checkpointing, and early stopping. \textit{PBS} exhibits a higher negative correlation with the F1 score compared to the Brier Score during training. Thus, \textit{PBS} more effectively identifies optimal checkpoints and early stopping points, leading to improved F1 scores. Comparative analysis verifies models selected by \textit{PBS} and \textit{PLL} achieve superior F1 scores. Therefore, \textit{PBS} and \textit{PLL} address the gap between uncertainty quantification and accuracy maximization by encapsulating both proper scoring principles and explicit preference for true classifications. The proposed metrics can enhance model evaluation and selection for reliable probabilistic classification.
\end{abstract}

% keywords can be removed

\section{Introduction}
Evaluation metrics play a critical role in model selection, feature selection, parameter tuning, and regularization when evaluating the performance of a classification model \cite{hossin2015review}.
In model selection, evaluation metrics such as accuracy, precision, recall, and F-measures are used to compare the performance of different models and select the best model for a specific task \cite{wang2020deep}.
Similarly, in feature selection, evaluation metrics are used to identify the most informative features for a specific task. By comparing the performance of a classification model with different feature subsets, irrelevant or redundant features can be eliminated, improving the model's efficiency and effectiveness \cite{pudjihartono2022review}.
Model checkpointing, also known as snapshotting, involves saving the state of a model during the training process at regular intervals. By evaluating the performance of each snapshot on validation metrics, the best-performing snapshot can be selected for the final model. This helps prevent overfitting by choosing a snapshot that exhibits good generalization ability on the validation set, improving the model's ability to generalize to unseen data \cite{zhang2020snapshot,annavarapu2021deep,siddiqui2023snapshot,griewank2000algorithm}.
In regularization, evaluation metrics are used to balance the model's complexity and its ability to generalize to new data. By evaluating the performance of a classification model with different regularization strengths, overfitting can be prevented, ensuring that the model performs well on unseen data \cite{canbek2022ptopi,prechelt2002early}.

In classification tasks, the evaluation of the model often relies on simple accuracy measures and related metrics derived from the confusion matrix that only considers predictions specifying the truth class \cite{tharwat2020classification}.
However, this approach overlooks the crucial aspect of probabilistic uncertainty quantification \cite{guo2017calibration}.
Ensuring the accuracy of predictive uncertainty is critical for classification models used in safety-critical applications. To be considered reliable and calibrated in statistical terms, a classification model must exhibit an honest expression of its predictive distribution \cite{vaicenavicius2019evaluating}.
In other words, the model should not only make predictions but also convey the level of uncertainty associated with those predictions \cite{xu2022deep}.
For example, consider a binary classification problem where the goal is predicting a certain disease.
A classifier that outputs a single class prediction of \textit{disease} or \textit{no disease} may achieve a high accuracy rate in terms of correct predictions. However, it fails to capture and communicate the inherent uncertainty surrounding each prediction.
This discrepancy between the use of simple accuracy measures in classification and the need for probabilistic assessments has motivated researchers to explore the use of scoring rules \cite{gneiting2005calibrated}.

Scoring rules are used to evaluate probabilistic predictions or forecasts in decision theory \cite{seidenfeld2012forecasting}.
When assessing probabilistic forecasts, it is crucial to use a scoring method that accurately reflects the reported probabilities and their alignment with actual outcomes \cite{winkler2019probability}. One approach to achieving this is through the use of strictly proper scoring rules. These rules assign optimal scores only when the forecasts match the true probabilities \cite{gneiting2007strictly}. The Brier Score, initially introduced in meteorology, is an early example of such a rule \cite{brier1950verification}. It was developed to discourage biased reporting and promote honest assessments of uncertainty. As a variant of the quadratic scoring rule, it remains widely utilized \cite{winkler2019probability}. Another pioneering effort resulted in the logarithmic rule, which is notable for its theoretical connection to entropy \cite{good1952rational}. Over time, numerous studies have examined commonly used scoring rules and their desirable statistical properties from different perspectives. Some analyses have focused on aspects such as forecast calibration and consistency incentives, while others have explored connections to information theory concepts \cite{guo2017calibration,carvalho2016overview}. As probabilistic prediction continues to be a significant area of research, ongoing investigations into existing rule properties and potential novel approaches will contribute to further understanding and advancement of scoring methodologies.
In this regard, this paper proposes a novel attribute for strictly proper scoring rules.

\subsection{Motivation}
The classification decision is commonly partitioned into categories of true positive, true negative, false positive, and false negative.
Let us examine an illustrative scenario involving a true vector of $[0, 1, 0]$ and two predicted vectors: $A = [0.33, 0.34, 0.33]$ and $B = [0.51, 0.49, 0]$. Vector $A$ corresponds to a true negative prediction, while vector $B$ represents a false negative prediction. It is crucial to critically evaluate which vector demonstrates superior performance, taking into account that the $\arg \max$ of vector $A$ accurately predicts the outcome, while vector $B$ assigns a higher probability to the correct class. Upon analyzing the performance of vectors $A$ and $B$ in a single-label scenario, it becomes evident that vector $A$ possesses an advantage over vector $B$ in terms of classification accuracy. This attribute holds significant importance as accurate decision-making is highly desirable and essential in numerous applications. Consequently, considering the paramount significance of precise decisions, it can be concluded that vector $A$ surpasses vector $B$ in quality.
Therefore, this paper advocates that the scoring rule ought to consistently assign better scores to decisions about true positives and true negatives.
However, a scientific discrepancy arises when considering the Brier Score and Logarithmic Loss metrics, as they occasionally contradict this advocated principle. Surprisingly, vector $A$ exhibits higher Brier Score and Logarithmic Loss values compared to vector $B$ contrary to expectations. This inconsistency between the advocated principle and the observed outcomes highlights a scientific gap or inconsistency that this paper aims to address and resolve. Therefore, this paper proposes a novel attribute for strictly proper scoring rules to tackle this issue.
Our contributions can be stated as the following:
\begin{itemize}
\item It hypothesizes that evaluation metrics should favor accurate predictions (true positives and true negatives) over misclassified ones (false positives and false negatives) when evaluating multi-class single-label classifications.

\item We show traditional strictly proper scores like Brier Score and Logarithmic Loss can rate incorrect predictions higher than correct ones.

\item A new property termed \textit{superior} is defined for scoring rules to formally establish that evaluations should preference true predictions.

\item To this end, this research proposes two novel scoring rules called Penalized Brier Score (PBS) and Penalized Logarithmic Loss (PLL).

\item This research modifies The scoring rules Brier Score and Logarithmic Loss by incorporating a penalty term to make them \textit{superior}. This penalizes incorrect predictions more strongly than the original scoring rules, aligning with the initial hypothesis that favors correct decisions over wrongs.

\item The utility of the proposed evaluation metrics for early stopping and checkpointing is experimentally assessed.

\item Experimental results demonstrate that the proposed evaluation metrics enhance the efficiency of the classifier. Furthermore, these metrics exhibit a high correlation with the F-1 score while encompassing the measurement of prediction uncertainty and calibration.

\end{itemize}
The paper is structured in the following way: Section \ref{sec:2} provides background information and sets the stage for the research. Section \ref{sec:3} reviews previous work related to the topic. Section \ref{sec:4} details the motivation for developing new scoring rules and introduces the proposed Penalized Brier Score and Penalized Logarithmic Loss methods. Section \ref{sec:5} presents experimental results obtained from spatio-temporal data. Finally, section \ref{sec:6} concludes the paper.

\section{Preliminaries}
\label{sec:2}

\subsection{Scoring Rules}
A scoring rule $S(\cdot,\cdot)$ is a bivariate function that takes two arguments: a probability measure, representing a forecast, and an observed outcome \cite{bolin2023local}.
The function returns a real number that measures the deviation between a forecast probability and reality.
They provide a summary measure for the evaluation of predictions that assign probabilities to events, such as probabilistic classification of a set of mutually exclusive outcomes or classes \cite{landes2015probabilism}.
Scoring rules can be thought of as \textit{cost function} or \textit{loss function} and are evaluated as the empirical mean of a given sample, simply called score \cite{gneiting2007strictly}.
It is worth noting that in neural networks, loss functions are typically utilized during the backpropagation process to optimize the network. However, in this research, we are focusing on evaluation metrics for assessing model performance.

Let $\Omega$ represent a general sample space, which contains all possible outcomes.
Let $\mathcal{A}$ represent an event space, which is a $\sigma$-algebra of subsets of $\Omega$.
An event refers to a set of observations within the sample space.
The extended real numbers $\overline{\mathbb{R}} = \left [ -\infty,\infty \right ]$ denote possible scores.
Let $\mathcal{F}$ represent a convex class of probability measures on the space $(\Omega, \mathcal{A})$.
A probabilistic forecast is any probability measure $F \in \mathcal{F}$.
%The scoring rules must be a measurable function, so they must be quasi-integrable (see, e.g., \cite{gneiting2007probabilistic,ferreira2011concept}).
%The quasi-integrability of the system ensures that the scoring rule is well-defined and can be used to evaluate classifiers in a meaningful way \cite{gneiting2007strictly}.
\begin{definition}[Scoring Rule]
\label{def:scoring_function}
%Let $(\Omega, \mathcal{A}, \mathcal{F})$ be a probability space.
A \textbf{scoring rule} $S:\mathcal{F} \times \Omega \rightarrow \overline{\mathbb{R}}$ is a function assigning numerical values to probability measure pairs $(P, i)$, where  $P \in \mathcal{F}$ as a probabilistic forecast is issued and  $i \in \Omega$ as an observation materializes.
\end{definition}

Let the distributional forecast $Q \in \mathcal{F}$ represent the forecaster's best estimate of the probability distribution.
In the context of multi-class classification tasks, $Q$ represents ground-truth vectors.
A scoring rule $S(P,Q)$ is defined as the expected score of $S(P,i)$ when $i \sim  Q$.
It means that
$S(P,Q) := \mathbb{E}_Q[ S(P,i) ] = \int S(P,i)dQ(i)$.
\begin{definition}[Strictly Proper Scoring Rule]
\label{def:strictly_proper}
The scoring rule is \textbf{proper} with respect to $\mathcal{F}$ if
\[
\left\{\begin{matrix}
S(P,Q) \geq S(Q,Q) & S~is~negatively~oriented \\ 
S(P,Q) \leq S(Q,Q) & S~is~positively~oriented
\end{matrix}\right.
\]
and it is \textbf{strictly proper} if equality holds if and only if $P=Q$.
\end{definition}

\subsection{Probabilistic Evaluation of Multi-Class Classification}
In single-label multi-class classification, the objective is to predict the specific class label, denoted as $\omega$, for a given sample characterized by a feature vector, denoted as $X$. 
%Mathematically, the class label $\omega$ is considered a random variable within the space $(\Omega, \mathcal{A})$, where $\Omega$ represents the set of classes with a cardinality of $c \in \mathbb{N}$. The feature vector $X$ is also a random vector that takes values in a feature space $\mathcal{X}$ contained within $\mathbb{R}^d$.
In probabilistic classification, the goal is to learn the entire conditional distribution $p(\omega|X)$, which represents the probability distribution over classes conditioned on the given features \cite{qian2021label}.
This is accomplished through the use of a probabilistic classifier denoted as $f : \mathcal{X} \rightarrow \mathcal{F}$.
%The classifier maps the features of a given instance to a probability measure over the set of classes.
The set $\mathcal{F}$ of probability measures is typically identified with the probability $c$-simplex $\mathcal{C} = \{ p \in [0,1]^c | \sum_{i=1}^c p_i = 1 \}$ and probability distributions are represented by vectors $p \in \mathcal{C}$.
%In other words, the probabilistic classifier $f(X)$ will return a probability vector $p$ with a probability for each of the $c$ outcomes.
Therefore, the overall form of the scoring rules for single-label multi-class classification problems is $S(p,i)$, where $p$ is a probability vector obtained from the classifier and $i$ is the true class.

\begin{definition}[]
\label{def:brier_score}
\textbf{Brier Score} (or squared loss or quadratic score), denoted by $S_{BS}(p,i)$, is the most well-known strictly proper scoring rule and is defined as \cite{brier1950verification}
\begin{align}
\label{eq:brier_score}
S_{BS}(p,i) = \sum_{j=1}^{c} (p_j - y_j)^2
\end{align}
Where $c$ in the number of classes and $y$ is the ground-truth label as a vector $y = (y_1 , \cdots , y_c )$ such that $y_i = 1$ if the true class is $i$, and $y_j= 0$ otherwise.
\end{definition}

\begin{definition}[]
\label{def:logarithmic_loss}
\textbf{Logarithmic Loss} (or Cross-entropy loss or log-loss or ignorance score), denoted by $S_{LL}(p,i)$, is a strictly proper scoring rules and is defined as \cite{good1952rational}
\begin{align}
\label{eq:logarithmic_loss}
S_{LL}(p,i) = - \sum_{j=1}^c y_j \log(p_j)
\end{align}
Where $c$ in the number of classes and $y$ is the ground-truth label as a vector $y = (y_1 , \cdots , y_c )$ such that $y_i = 1$ if the true class is $i$, and $y_j= 0$ otherwise.
\end{definition}

%Probabilistic classifiers can evaluated using $S_{BS}$ or $S_{LL}$.Both scoring rules are minimized by the true posterior class probabilities produced by the Bayes-optimal model \cite{kull2015novel,gneiting2007strictly}. Since this work focuses on the Brier Score and Logarithmic Loss, which should be minimized, framing the discussion and analysis around minimizing scores avoids potential confusion compared to maximizing scores, as is common for positively oriented rules like accuracy and F-measures.

In practice, competing forecast procedures are ranked by the average score \cite{gneiting2007strictly}.
Suppose that we wish to fit a parametric model $P_{\theta}$ based on random samples $X^{(1)}, \cdots , X^{(n)}$.
To estimate $\theta$, we might measure the goodness-of-fit by the mean score \cite{dawid2014theory}
\begin{align}
\label{eq:mean_score}
S_n(\theta) = \frac{1}{n} \sum_{i=1}^{n} S(P_{\theta}(X^{(i)}),\omega^{(i)})
\end{align}
where $S$ is a strictly proper scoring rule.
If $\theta_0$ denotes the true parameter, asymptotic arguments indicate that $\arg \min_\theta$ $S_n(\theta) \rightarrow \theta_0$ as $n \rightarrow \infty$ \cite{gneiting2007strictly}.
Let's suppose $S$ is a negatively oriented scoring rule like the Brier Score.
It suggests a general approach to estimation: choose a strictly proper scoring rule that is tailored to the problem at hand that minimizes $S_n(\theta)$ over the parameters space and take $\hat{\theta}_n = \arg \min_\theta S_n(\theta)$ as the optimum score estimator based on the scoring rule \cite{gneiting2007strictly}.
For positively oriented scoring rules like accuracy, $\hat{\theta}_n = \arg \max_\theta S_n(\theta)$.

\section{Literature Review}
\label{sec:3}
Evaluation metrics are crucial for assessing the performance of classification algorithms. 
They are used to evaluate the performance and effectiveness of classifiers, as well as to discriminate and select the optimal solution during the classification training process \cite{hossin2015review}.
Classification evaluation metrics can generally be divided into accuracy-based and confidence-based metrics.

Accuracy-based metrics only consider whether a prediction matches the true label, without regard to predictive confidence or uncertainty.
This category includes metrics like accuracy, error rate, precision, recall, F-measures, AUC, and related threshold-based measures.
Accuracy and error rate are among the most widespread evaluation metrics.
However, they have limitations such as providing less distinctive values and favoring the majority class \cite{mortaz2020imbalance,tharwat2020classification}.
Precision and recall focus on different targets, making them unsuitable for selecting the optimal solution \cite{hossin2015review}.
F-measures and AUC aim to address some of the shortcomings of individual metrics.
F-measures consider both precision and recall, while AUC reflects the overall ranking performance of a classifier \cite{mortaz2020imbalance}.
However, AUC has a high computational cost for large datasets \cite{hossin2015review}. 

Confidence-based metrics assess predictive confidence by rewarding calibrated probability estimates \cite{kull2015novel}.
These metrics reward probabilistic calibration and include strictly proper scoring rules.
In addition to the scoring rules introduced in the preliminaries, there are several other scoring rules that have been utilized in forecast verification \cite{wilks2011forecast}. Some examples include the spherical score, quadratic score, pseudospherical score, and continuous ranked probability score (CRPS) \cite{hersbach2000decomposition,tyralis2024review}.
Variations of these rules have been used in different fields. This includes electricity price forecasting to assess volatility over time, and wind/weather modeling to evaluate probabilistic predictions \cite{nowotarski2018recent,gaetan2024calibrated}. Financial prediction and spatial statistics also apply scoring rules \cite{tyralis2024review}.

%\subsection{Benefits of Strictly Proper Scoring Rules}
A scoring rule is considered proper if the expected score is minimized by the true distribution of the outcome of interest \cite{resin2023classification}.
This means that a forecaster who uses a proper scoring rule will be incentivized to provide a forecast that is as close as possible to the true distribution.
In terms of elicitation, scoring rules have a significant role in motivating assessors to provide conscientious and truthful assessments \cite{gneiting2007strictly}.
Nonetheless, because the concept of strict propriety adds an additional level of uniqueness, the utilization of strictly proper scoring rules offers enhanced theoretical assurances regarding the optimality of truth-telling as a strategy \cite{gneiting2007strictly}.
This encourages predictors to provide more honest probabilistic forecasts.

In addition, assessing forecast characteristics like calibration and sharpness is important for understanding individual and combined forecasts \cite{winkler2019probability}.
Calibration examines the consistency between forecasts and outcomes, while sharpness measures the concentration of the forecast distribution for a variable of interest, regardless of outcomes \cite{gneiting2007strictly}.
The relationship between strictly proper scoring rules and measures of calibration and sharpness can be established by decomposing the rules into components.
A common decomposition enables the expression of a scoring rule as a function of both a calibration measure and a sharpness measure \cite{winkler2019probability,kull2015novel}.
By utilizing strictly proper scoring rules, the evaluation process can capture the probabilistic uncertainty inherent in classification problems and prevent overconfidence \cite{winkler2019probability}.
Accuracy and F-measures are two proper scoring rules, which are positively oriented. They are not strictly proper scoring rules because their maximum are not unique \cite{gneiting2007strictly}.
Therefore, strictly proper scoring rules such as Brier Score and Logarithmic Loss have more advantages than metrics derived from the confusion matrix.
In conclusion, this paper recommends the evaluation of multi-class classification using strictly proper scoring rules for the following reasons:
\begin{itemize}
\item \textbf{Elicitation}: Strictly proper scoring rules have desirable elicitation properties. They uniquely incentivize assessors to provide honest probabilistic forecasts by maximizing the expected score \cite{gneiting2007strictly}.
\item \textbf{Reliability}: Strictly proper scoring rules enable reliable assessments of uncertainty. Decomposing proper scoring rules into calibration and sharpness components facilitates the evaluation of forecast uncertainty. This prevents overconfidence by capturing the inherent probabilistic uncertainty in classification \cite{winkler2019probability,kull2015novel}.
\end{itemize}

%\subsection{Drawback of Strictly Proper Scoring Rules}
While this paper recommends using strictly proper scoring rules, they are not always sufficient to ensure that more accurate forecasts receive better scores.
Previous evaluation metrics for multi-class classification tasks have the shortcoming of not explicitly favoring accurate predictions over misclassified ones. Metrics like Brier Score and Logarithmic Loss, while being strictly proper scoring rules, do not consistently assign better scores to predictions that result in true positives and true negatives compared to those that are false positives and false negatives. This can potentially lead to suboptimal model selection if a model achieves better scoring by making incorrect predictions rather than correct ones. The paper addresses this shortcoming by proposing novel superior scoring rules called Penalized Brier Score and Penalized Logarithmic Loss.

\section{Proposed Model}
\label{sec:4}
This section will first present the theoretical background, providing theoretical results that motivate the requirement for proper scoring rules that assign higher scores to correct predictions rather than incorrect ones. Next, the proposed methodology will be introduced. It will illustrate how the scoring rules can be modified by incorporating penalty terms. This serves to penalize incorrect predictions more strongly, thereby ensuring the scoring rules consistently assign higher scores to correct predictions as intended.

\begin{figure}[http]
	\centering
	\includegraphics[width=1.0\textwidth]{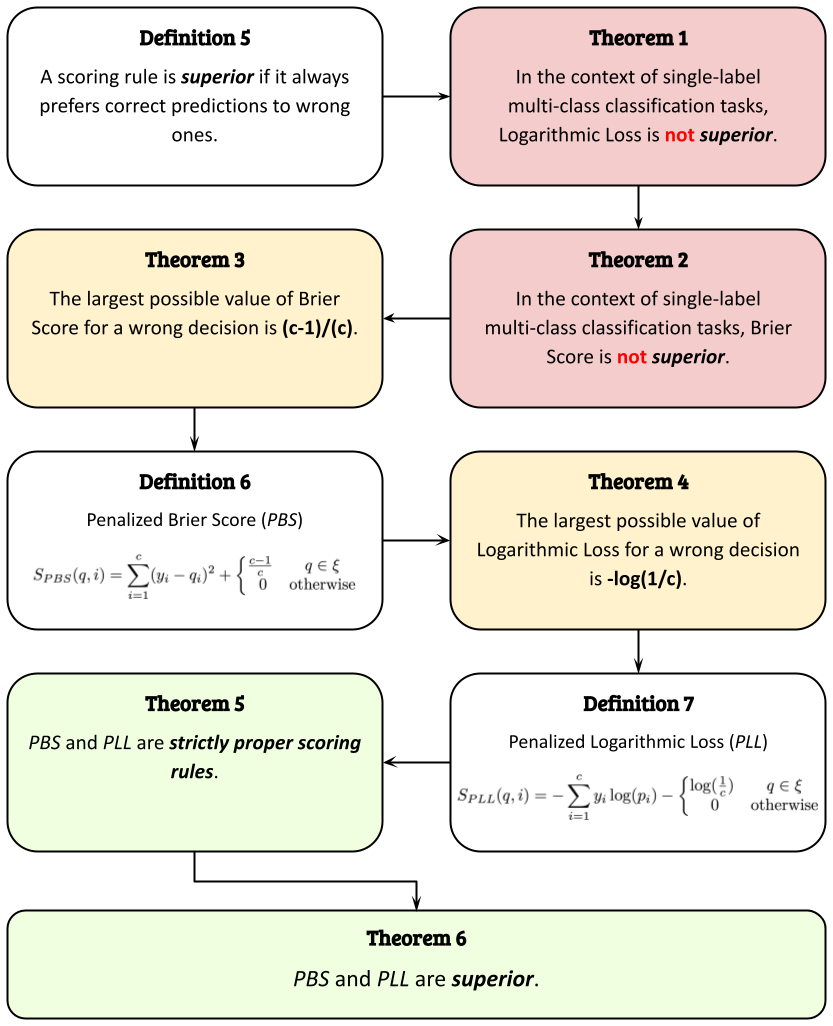}
	\caption{The big picture of theoretical analysis.}
	\label{fig:bigpic}
\end{figure}

\subsection{Theoretical Analysis}
The primary objective of any classifier is to accurately assign observations to their respective classes. Observations that are correctly classified are considered to be of greater value than those that are incorrectly classified. A classifier that can achieve a high rate of accurate classifications and a low error rate is deemed superior because it indicates a greater ability to identify patterns and differentiate between classes. Conversely, observations that are incorrectly classified suggest that the classifier has not accurately modeled the relationship between features and classes. These errors compromise the classifier's capacity to generalize and make accurate predictions on new samples. Correct classifications represent successful outcomes for the classifier, while errors indicate failures. Therefore, it is crucial to maximize correct classifications and minimize errors in order to develop an effective and useful classifier. The reliability of a classifier's performance is determined by its ability to consistently reproduce the correct labels instead of settling for inferior results that are prone to mistakes.

Fig. \ref{fig:bigpic} shows the big picture of the theoretical analysis presented in this section. The subsequent part of this section demonstrates that both Brier Scoring and Logarithmic Loss are indifferent to this preference. To prove this point, a novel attribute for scoring rules is introduced. This attribute guarantees that the scoring rule assigns a superior score to the observations that are correctly classified.

\begin{definition}[Superior Scoring Rule]
\label{def:3conditions}
%Let $(\Omega, \mathcal{A}, P)$ be a probability space, where $\Omega = \{1,\cdots,c \}$ represents the set of classes.
Let $\psi$ and $\xi$ be two subsets of the set of all predictions by a certain classifier:
\begin{align*}
\psi &= \left\{ p ~|~ \arg \max p = \arg \max y \right\},
\\
\xi  &= \left\{ p ~|~ \arg \max p \neq \arg \max y \right\}.
\end{align*}
Where $y$ is the ground-truth label as a vector $y = (y_1 , \cdots , y_c )$ such that $\sum_{i=1}^c y_i = 1$ and $y_i = 1$ when the $i$ is the true class.
Therefore, $\psi$ contains true positive and true negative predictions, and $\xi$ contains false positive and false negative predictions.
Let $S(p, \omega)$ represent the score when the prediction $p \in \mathcal{P}$ is issued and the true class is $\omega$.
$S$ is \textbf{superior} if its scores for every member belonging to $\psi$ are always better than those of set $\xi$.
Let $p^{(1)} \in \psi$ with its corresponding true class $\omega^{(1)}$.
Let $p^{(2)} \in \xi$ with its corresponding true class $\omega^{(2)}$.
Then, if the condition
\[
\left\{\begin{matrix}
S(p^{(1)},\omega^{(1)}) < S(p^{(2)},\omega^{(2)}) & S~is~negatively~oriented \\ 
S(p^{(1)},\omega^{(1)}) > S(p^{(2)},\omega^{(2)}) & S~is~positively~oriented
\end{matrix}\right.
\]
always satisfied, $S$ is \textbf{superior}.
\end{definition}
Therefore, if a scoring rule has this property, it will always give better scores to correct predictions. In the following, we show theoretically that Logarithmic Loss and Brier Score are not superior.
% Theorem 1
\begin{theorem}[LL Property]
\label{thrm:LL_Property}
In the context of single-label multi-class classification tasks, $c>2$, Logarithmic Loss is not \textbf{superior}.
\end{theorem}
\begin{proof}
See Appendix \ref{proof:LL_Property}.
\end{proof}

% Theorem 2
\begin{theorem}[BS Property]
\label{thrm:BS_Property}
In the context of single-label multi-class classification tasks, $c>2$, Brier Score is not \textbf{superior}.
\end{theorem}
\begin{proof}
See Appendix \ref{proof:BS_Property}.
\end{proof}

According to the Monte Carlo simulations provided in Appendix \ref{proof:BS_Property}, it is shown that Brier Score outperforms Logarithmic Loss. However, to ensure that the scoring rule consistently assigns a higher score to accurately classified observations, a new criterion is necessary, which can be attained by modifying the current scoring rules. Specifically, if the probability vector belongs to set $\xi$, an extra penalty was integrated into the scoring rule.
Since this research focuses on the Brier Score and Logarithmic Loss, which should be minimized to be optimal, our aim is to assign a penalty to incorrect predictions that the penalty value is higher than the highest score possible for a correct prediction. To accomplish this, the penalty should be established as equivalent to the maximum value of the scoring rule for set $\psi$.

% Theorem 3
\begin{theorem}[Maximum $S_{BS}$ in $\psi$] \label{thrm:MaxBS}
The largest possible value of the Brier Score for the set $\psi$ is
$\frac{c-1}{c}$.
\end{theorem}
\begin{proof}
See Appendix \ref{proof:MaxBS}.
\end{proof}

% Theorem 4
\begin{theorem}[Maximum $S_{LL}$ in $\psi$] \label{thrm:MaxLL}
The largest possible value of Logarithmic Loss for the set $\psi$ is
$- \log (\frac{1}{c})$.
\end{theorem}
\begin{proof}
See Appendix \ref{proof:MaxLL}.
\end{proof}

Based on the highest possible scores for the Brier Score and Logarithmic Loss when predictions are incorrect, an adjusted form of these scoring rules can be defined.
The modified Brier Score with the penalty term, Penalized Brier Score (\textit{PBS}), can be expressed as:
\begin{align}
\label{PBS}
S_{PBS}(q,i) = \sum_{i=1}^{c}(y_i-q_i)^2 + 
\left\{\begin{matrix}
\frac{c-1}{c} & q \in \xi\\\ 
0 & \text{otherwise}
\end{matrix}\right.
\end{align}
The modified Logarithmic Loss with the penalty term, Penalized Logarithmic Loss (\textit{PLL}), can be expressed as:
\begin{align}
\label{PLL}
S_{PLL}(q,i) = - \sum_{i=1}^{c} y_i \log(p_i) - 
\left\{\begin{matrix}
\log (\frac{1}{c}) & q \in \xi\\\ 
0 & \text{otherwise}
\end{matrix}\right.
\end{align}
where $y$ is the ground-truth vector, $q$ is the predicted probability vector by the probabilistic classifier, and $c$ is the number of classes.
As shown in the following theorems, the two proposed evaluation metrics are not only strictly proper scoring rules but also superior.

% Theorem 5
\begin{theorem}[\textit{PBS} \& \textit{PLL} Scoring Rules] \label{thrm:PBSPLLScoringRules}
\textit{PBS} and \textit{PLL} are strictly proper scoring rules.
\end{theorem}
\begin{proof}
See Appendix \ref{proof:PBSPLLScoringRules}.
\end{proof}

% Theorem 6
\begin{theorem}[\textit{PBS} \& \textit{PLL} Property] \label{thrm:PBSPLLProperty}
\textit{PBS} and \textit{PLL} are \textbf{superior}.
\end{theorem}
\begin{proof}
The penalty term represents the maximum score that can be assigned to a correct prediction. Consequently, the scoring rule modified by the penalty term passively assigns higher scores to correct predictions over incorrect ones in a consistent manner.
\end{proof}

Penalties in Eq. (\ref{PBS}) and Eq. (\ref{PLL}) are designed to be applied only to incorrect predictions.
Therefore, these penalties depend on the values of $p$ in Eq. (\ref{eq:brier_score}) and Eq. (\ref{eq:logarithmic_loss}).
Furthermore, the more number of incorrect predictions by $\theta$ in Eq. (\ref{eq:mean_score}), the worse score it receives.
As a result, the modified evaluation metrics can more reliably identify optimal model checkpoints and early stopping points that achieve better generalization performance compared to the original metrics like Brier Score and Logarithmic Loss.

\subsection{Algorithm}
In this section, we introduce the vectorization methods of the proposed evaluation metrics.
The pseudocode of the \textit{Penalized Brier Score} and the \textit{Penalized Logarithmic Loss} is demonstrated in Algorithm \ref{alg:PBS} and Algorithm \ref{alg:PLL}, respectively.
The \textit{PBS} and \textit{PLL} algorithms leverage Penalizing Algorithm \ref{alg:penalizing} to incorporate penalties into the Brier Score and the Logarithmic Loss.

The Penalizing algorithm detects wrong predictions and penalizes them. The output is in the form of a vector, where $0$ means that the prediction was correct, otherwise the prediction was wrong and was penalized. This payoff vector is then used by other scoring functions to calculate performance metrics.
It takes as input predicted probabilities $q \in [0,1]^{n \times c}$, ground-truth labels $y \in \{0,1\}^{n \times c}$, and a \textit{penalty} value for wrong predictions. It first computes the \textit{hot values} for each sample, which are the predicted probabilities for the correct class. It then subtracts the \textit{hot values} from the predicted probabilities to obtain a \textit{container} of values that are positive for incorrect predictions. All negative values in the \textit{container} are set to zero, and the sum of the remaining values is computed for each sample.
The \textit{penalty} value is then multiplied by a vector of \textit{ones} to become a vector. Finally, all positive values in the \textit{container} are replaced by \textit{penalty} and the \textit{container} is returned as the \textit{payoff} for each sample.

The \textit{PBS} algorithm takes as input predicted probabilities $q$ and ground-truth labels $y$. It first computes the \textit{penalty} factor as $(c - 1) / (c^2)$, where $c$ is the number of classes. It then calls the \textit{Penalizing} algorithm to compute the \textit{payoff} for each sample. Next, the Brier scores are computed as the mean squared difference between the predicted probabilities and the ground-truth labels. Finally, the \textit{PBS} is computed as the mean of the sum of the Brier scores and the \textit{payoffs}.

The \textit{PLL} algorithm is similar to the \textit{PBS} algorithm, but it uses the Logarithmic Loss as the base score instead of the Brier Score. The Logarithmic Loss is computed as the mean of the sum of the element-wise product of the ground-truth labels and the negative logarithm of the predicted probabilities. Finally, the \textit{payoff} is subtracted from the Logarithmic Loss before taking the mean.

\begin{algorithm}
\caption{Penalizing}
\label{alg:penalizing}
\begin{algorithmic}[1]  

\Require {
\begin{itemize}
\item $\textbf{q}$: Predictions of size $n \times c$,
\item $\textbf{y}$: Ground-truth labels of size $n \times c$, where $n$ is the number of samples and $c$ is the number of classes. 
\item $\textbf{penalty}$: A scalar for wrong predictions.
\end{itemize}
}

\State Consider operator $==$ applied to an array means applying element-wise $==$ to the array. The result is a condition array consisting of $True$ and $False$.

\State Consider function $\mathsf{where}(condition,x,y)$ gives three arrays with the same shape.
The result is taken from $x$ where $condition$ is $True$ or $y$ where it is $False$.

\State Consider function $\mathsf{sum}(x,axis)$ 
computes the sum of the $x$ elements over $axis$.
The result is an array with the same shape as $x$, with the specified axis removed.

\State Consider function $\mathsf{mean}(x,axis)$ 
computes the average of the $x$ elements over $axis$.
The result is an array with the same shape as $x$, with the specified axis removed.

\State Consider function $\mathsf{zeros}(x,y)$ return a zero array with shape $x \times y$. And function $\mathsf{ones}(x,y)$ return one array with shape $x \times y$. If y is not provided the result would be 1-dimension.

\State $\mathsf{hotvalues} \gets \mathsf{sum}(\mathsf{where}(y == 1, q, y),axis=1)$

\State $\mathsf{container} \gets q - \mathsf{hotvalues}\mathit{1}_c^T$

\State $\mathsf{zeros} \gets \mathsf{zeros}(n,c)$

\State $\mathsf{container} \gets \mathsf{where}(\mathsf{container}<0,\mathsf{zeros},\mathsf{container})$ 

\State $\mathsf{container} \gets \mathsf{sum}(\mathsf{container}, \mathsf{axis}=1)$

\State $\mathsf{penalty} \gets \mathsf{ones}(c) \cdot \mathsf{penalty} $

\State $\mathsf{payoff} \gets \mathsf{where}(\mathsf{container} > 0, \mathsf{penalty}, \mathsf{container})$

\State \Return $\mathsf{payoff}$

\end{algorithmic}
\end{algorithm}

\begin{algorithm}
\caption{Penalized Brier Score}
\label{alg:PBS}
\begin{algorithmic}[1]  

\Require {
Probability measures $\textbf{q}$ of size $n \times c$,
Ground-truth labels $\textbf{y}$ of size $n \times c$, where $n$ is the number of samples and $c$ is the number of classes
}

\State $\mathsf{penalty} \gets (c - 1) / (c^{2})$

\State $\mathsf{payoff} \gets \mathsf{Penalizing}(q, y, \mathsf{penalty})$

\State $\mathsf{bs} \gets \mathsf{mean}(\mathsf{square}(y - q), \mathsf{axis}=1)$ 
\algorithmiccomment{Brier Scores}

\State \Return $\mathsf{mean}(\mathsf{bs} + \mathsf{payoff})$

\end{algorithmic}
\end{algorithm}

\begin{algorithm}
\caption{Penalized Logarithmic Loss}
\label{alg:PLL}
\begin{algorithmic}[1]

\Require {Probability measures $\textbf{q}$ of size $n \times c$,
Ground-truth labels $\textbf{y}$ of size $n \times c$, where $n$ is the number of samples and $c$ is the number of classes}

\State $\mathsf{penalty} \gets \log(1/c)$

\State $\mathsf{payoff} \gets \mathsf{Penalizing}(q, y, \mathsf{penalty})$

\State $\mathsf{ll} \gets \mathsf{mean}(\mathsf{sum}(y \cdot \log(q)))$

\State \Return $\mathsf{mean}(\mathsf{ll} - \mathsf{payoff})$

\end{algorithmic}
\end{algorithm}

\section{Experimental results}
\label{sec:5}
\subsection{Datasets}
Scoring rules are commonly used in the field of spatio-temporal statistics to compare different models \cite{heaton2019case}. To assess the effectiveness of our proposed method, we have selected several spatio-temporal datasets. Classifying spatio-temporal data is generally challenging, as evidenced by the low accuracy scores achieved on these datasets. For this reason, thorough model evaluation takes on greater importance.
Additionally, it is crucial to consider incorrect classes when the classifier's ability to generalize is limited. While the probabilities related to incorrect classes may not be particularly informative if the probability of the true class for all observations is above 0.5, it becomes increasingly important when the probability of the true class is below 0.5 for many observations. As a result, any evaluation method must take into account incorrect classes when making decisions.

The proposed model was evaluated using various classification problems, as described below. The datasets used for evaluation included three-axis accelerometer signals obtained from participants' activities such as running or lying, as described in \cite{casale2012personalization,weiss2019smartphone,torres2013sensor,kaluvza2010agent}. While these datasets were intended for research purposes related to activity recognition, they also presented challenges for identifying individuals based on their motion patterns.
Also, the dataset presented in \cite{scalabrini2019prediction} was generated for predicting motor failure time using three-axis accelerometer data. 
The dataset in \cite{salam2018comparison} provided information about power consumption such as temperature, humidity, wind speed, consumption, general diffuse flows, and diffuse flows in three zones of Tetouan City, which is suitable for predicting a zone based on consumption information.
The dataset presented in \cite{zhang2017cautionary} includes hourly air pollutant data such as PM2.5, PM10, SO2, NO2, CO, O3, temperature, pressure, dew point temperature, precipitation, wind direction, wind speed from 12 air-quality monitoring sites, and the proposed model predicted the sites based on their air-quality information. The location of participants was collected using two received signal strength indicator (RSSI) measurements, as described in \cite{Dua:2019}. While the primary task was indoor localization, the proposed model aimed to identify participants based on their location.
Finally, the dataset presented in \cite{eftekhari2018hybrid} included three-axis accelerometer and three-axis gyroscope data from ten drivers, and the goal was to identify drivers based on their driving behaviors.

\begin{figure}[http]
	\centering
	\includegraphics[width=1\textwidth]{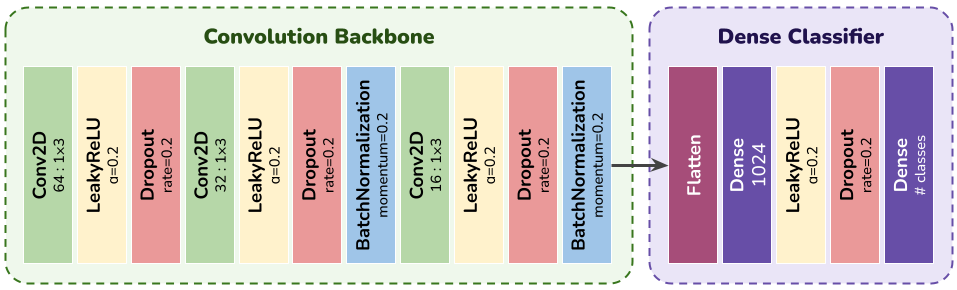}
	\caption{The architecture of the CNN classifier.}
	\label{fig:CNN}
\end{figure}

\subsection{Evaluation Strategy}
Since the purpose of this evaluation is to analyze two proposed metrics, the classification model used is a Convolutional Neural Network (CNN).
CNNs are widely used for spatial and temporal data modeling tasks \cite{wang2020deep}.
The Fig. \ref{fig:CNN} illustrates the architecture of the proposed CNN model.
As neural networks are optimized through iterative training procedures, the proposed metrics can be systematically tested during this process.
Model checkpointing saves model weights periodically during training. This allows rolling back to previous checkpoints if overfitting or divergence occurs. Early stopping monitors validation performance and stops training if no improvement is seen for a set number of epochs, preventing overfitting.
Implementing these techniques with the proposed metrics provides insights into their behavior at different stages of neural network optimization.

\begin{algorithm}
    \caption{h-Block Cross-Validation}
    \label{alg:h-block-cv}
    \begin{algorithmic}[1]
        \Procedure{h-BlockCrossValidation}{$data$, $h$}
            \State Divide $data$ into $h$ non-overlapping blocks: $B=\{B_1, B_2, \ldots, B_h\}$
            \State Validation Size $nv \gets \left \lfloor 0.2h \right \rfloor$
            \State Test Size $nt \gets \left \lfloor 0.3h \right \rfloor$
            \For{$i = 1$ to $h$}
            		\State Validation block $B_v \gets B[i:i+nv]$
            		\State Test block $B_{ts} \gets B[i+nv:i+nv+nt]$
                \State Train block $B_{tr} \leftarrow B \setminus (B_v \cup B_{ts})$
                \State Segmenting blocks by sliding window
                \State Train and validate model on $B_{tr}$ and $B_{v}$
                \State Evaluate model on $T_{ts}$
            \EndFor
        \EndProcedure
    \end{algorithmic}
\end{algorithm}

When dealing with spatio-temporal data, it is essential to employ a cross-validation technique that accounts for the temporal dependencies within the data. This is crucial because the data exhibits interdependence, and employing a random sample selection for training and testing could introduce biases in the results \cite{shao1993linear}.
$h$-block cross-validation is a type of temporal cross-validation that is often used for temporal data \cite{shao1993linear}.
It helps address the issue of data leakage that can occur in standard $k$-fold cross-validation when applied to time series. The Algorithm \ref{alg:h-block-cv} shows the pseudocode of $h$-block CV. The time series data was first divided into $h$ non-overlapping blocks, each representing a continuous time interval.
For each block $i$, a validation set $B_v$ is created by selecting a contiguous subset of $nv$ blocks starting from block $i$. A test set $B_{ts}$ is then created by selecting a contiguous subset of $nt$ blocks starting from the end of the validation block. The remaining blocks are assigned to the training set $B_{tr}$.
Within each block, windowing techniques were applied to segment data.
The model is trained on the segments of the training set $B_{tr}$ and validated on the segments of the validation set $B_v$. The trained model is then evaluated on the segments of the test set $B_{ts}$.
For each iteration, 50\% of the subsets are allocated for training, 30\% for testing, and 20\% for validation. By adopting this approach, the model is trained on a sufficiently large dataset that enables effective generalization to new data, while the testing phase utilizes an independent dataset.

The validation set helps guide hyperparameter tuning, such as selecting the optimal window length and overlap values of the sliding window evaluated via grid search.
Table \ref{tbl:datasets} summarizes the key attributes of each dataset used in the experiments.
It also allows assessing model checkpointing and early stopping criteria by tracking validation performance over training epochs. Once hyperparameters are fixed, the final model performance is reported on the independent test sets from each fold. This rigorous evaluation process provides robust estimates of the proposed metrics.
The CNN model is implemented using \textit{Python} and the \textit{TensorFlow} library. \textit{Nadam} is employed to optimize network parameters. The results of the experiments can be accessed via \href{https://github.com/Ruhallah93/superior-scoring-rules}{GitHub}.

\begin{table}[http]
\scriptsize
\centering
\renewcommand{\arraystretch}{1.2}
\caption{Datasets were used in evaluating the proposed method.}
\label{tbl:datasets}

\begin{tabular}{|c|c|c|c|}
\hline
\textbf{Dataset}                 & \textbf{Number of Classes} & \textbf{Window Length} & \textbf{Window Overlap} \\ \hline
\cite{casale2012personalization} & 13           &00:00:03&              75\%  \\ \hline
\cite{weiss2019smartphone}       & 10           &00:00:08&              75\%  \\ \hline
\cite{torres2013sensor}          & 12           &00:00:06&              75\%  \\ \hline
\cite{kaluvza2010agent}          & 5            &00:00:03&              75\%  \\ \hline
\cite{scalabrini2019prediction}  & 3            &00:00:06&              75\%  \\ \hline
\cite{salam2018comparison}       & 3            &00:07:00&              75\%  \\ \hline
\cite{zhang2017cautionary}       & 12           &00:01:00&              75\%  \\ \hline
\cite{Dua:2019}                  & 12           &00:00:30&              75\%  \\ \hline
\cite{eftekhari2018hybrid}       & 10           &00:00:10&              75\%  \\ \hline
\end{tabular}

\end{table}

\subsection{Details of Experiments}
To ensure a thorough evaluation of the proposed evaluation metrics, the model training incorporated model checkpointing (\textit{CP}) and early stopping (\textit{ES}) techniques, considering both traditional evaluation metrics and the proposed metrics. \textit{CP} were saved during training at epochs where either the traditional metrics (such as Brier Score and Logarithmic Loss) or the proposed metrics demonstrated the best performance on the validation set. \textit{ES} was implemented as well, terminating training if there was no improvement in either the traditional metrics or the proposed metrics for a specified number of epochs, thus preventing overfitting. To enhance the robustness of the results, this entire process was repeated $100$ times to identify the optimal model hyperparameters. Such an approach enabled a fair comparison between the proposed metrics and the sole use of traditional metrics in determining the best model checkpoint. Subsequently, the checkpoint yielding the highest performance, as indicated by each metric, was evaluated on the test set, providing an accurate assessment of the proposed metric's ability to identify the model with the highest true performance.

The F1 score and accuracy are widely recognized as a prominent evaluation metric for classification tasks.
To effectively evaluate model performance on such unbalanced datasets, accuracy alone is not a suitable metric \cite{grina2023re}. 
As most of the real-world datasets are either heavily or moderately imbalanced, accuracy can be misleading when the number of samples is very different between classes \cite{moral2022using}.
For this reason, the macro-averaged F1 score is also reported. F1 score considers both precision and recall, providing a better sense of classification effectiveness on unbalanced problems \cite{hossin2015review}.
However, it is important to note that this metric solely serves as a proper scoring rule and does not encompass the measurement of forecast uncertainty and overconfidence \cite{guo2017calibration}. In order to address this limitation, it is proposed that \textit{strictly proper scoring rules} be employed, as they possess the capability to gauge uncertainty. By exhibiting behavior similar to the F1 score while also measuring uncertainty, these proposed scoring rules can effectively fulfill both reliability and accuracy requirements. Therefore, incorporating such scoring rules in the evaluation process would provide a comprehensive assessment of classification models, encompassing not only their predictive accuracy but also their ability to quantify uncertainty.

\begin{figure}[http]
\centering     %%% not \center
\subfloat[\cite{eftekhari2018hybrid}]{\label{fig:a}\includegraphics[width=0.33\textwidth]{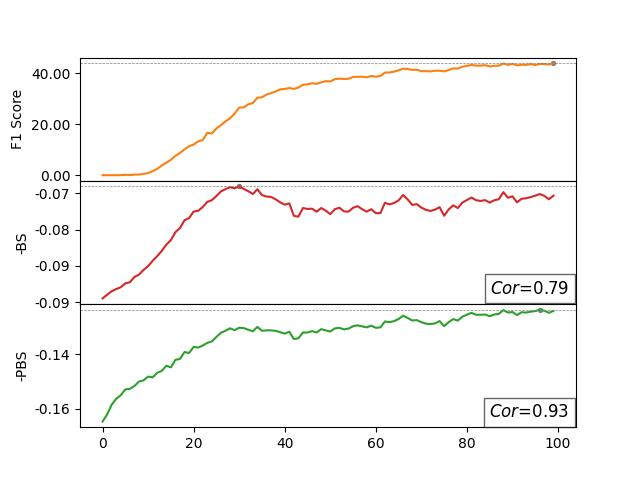}}
\subfloat[\cite{kaluvza2010agent}]{\label{fig:b}\includegraphics[width=0.33\textwidth]{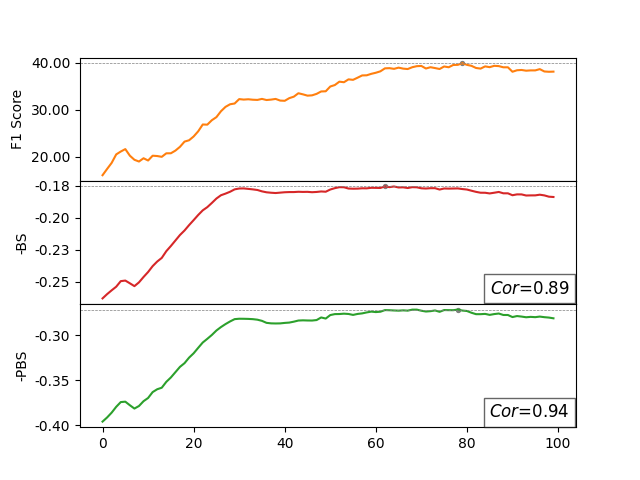}}
\subfloat[\cite{torres2013sensor}]{\label{fig:b}\includegraphics[width=0.33\textwidth]{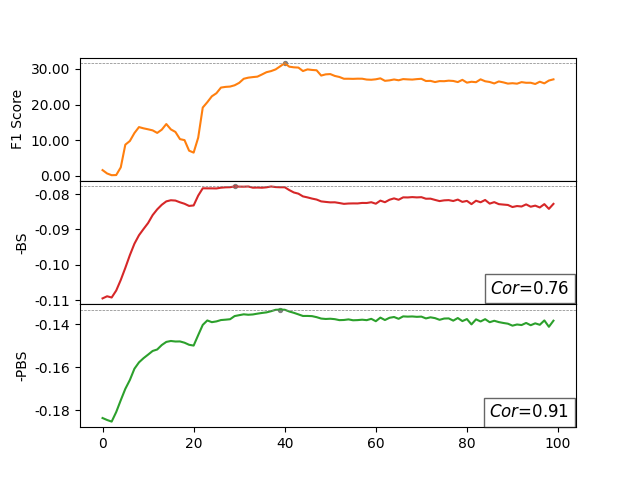}}
\\
\subfloat[\cite{scalabrini2019prediction}]{\label{fig:a}\includegraphics[width=0.33\textwidth]{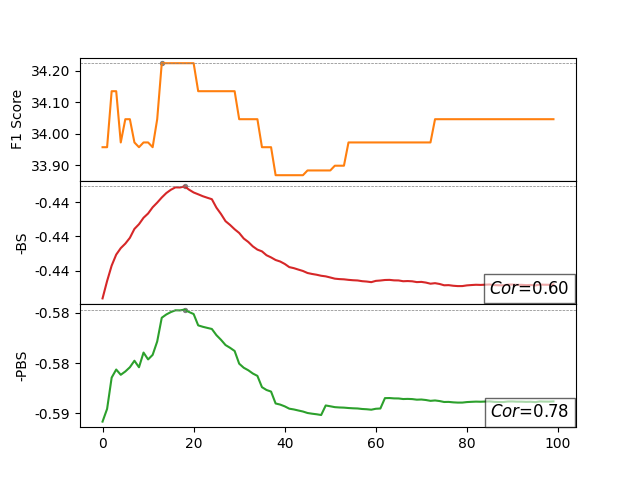}}
\subfloat[\cite{salam2018comparison}]{\label{fig:b}\includegraphics[width=0.33\textwidth]{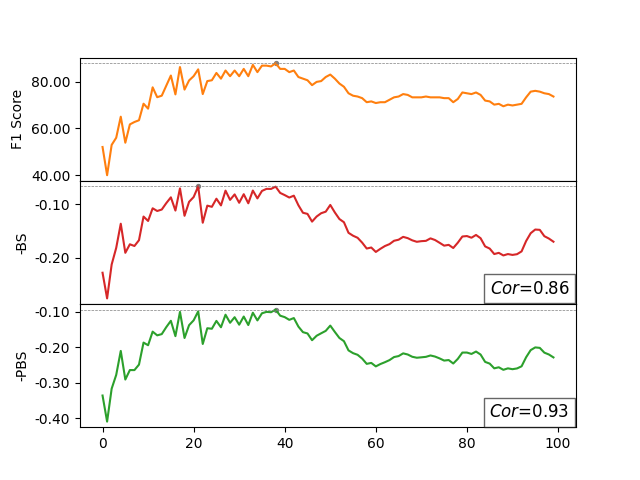}}
\subfloat[\cite{zhang2017cautionary}]{\label{fig:b}\includegraphics[width=0.33\textwidth]{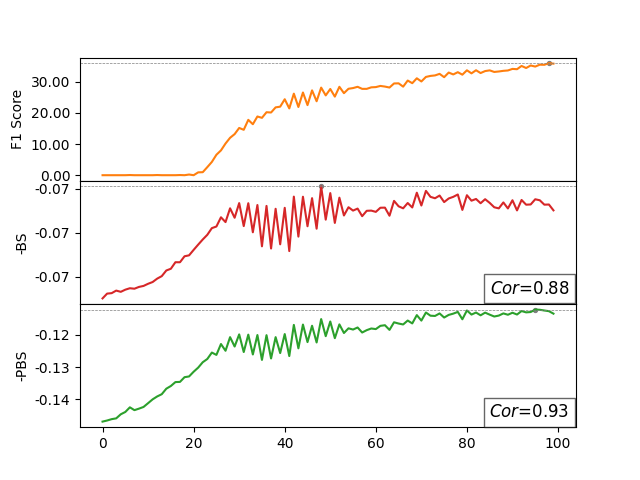}}
\\
\subfloat[\cite{Dua:2019}]{\label{fig:a}\includegraphics[width=0.33\textwidth]{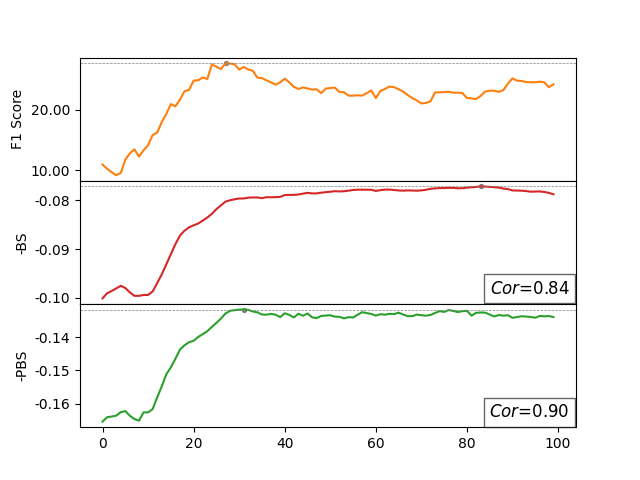}}
\subfloat[\cite{casale2012personalization}]{\label{fig:b}\includegraphics[width=0.33\textwidth]{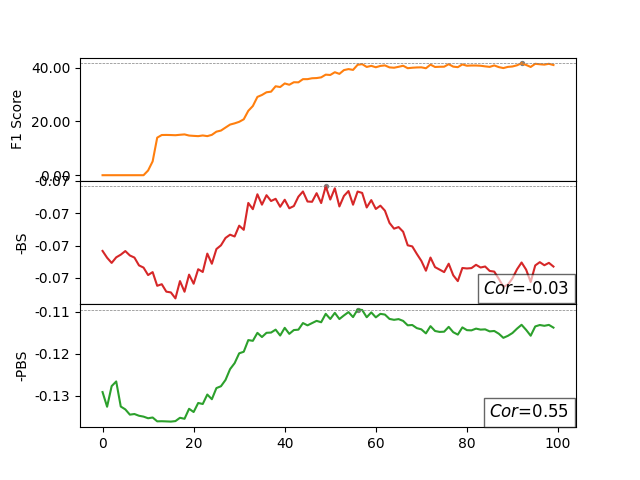}}
\subfloat[\cite{weiss2019smartphone}]{\label{fig:b}\includegraphics[width=0.33\textwidth]{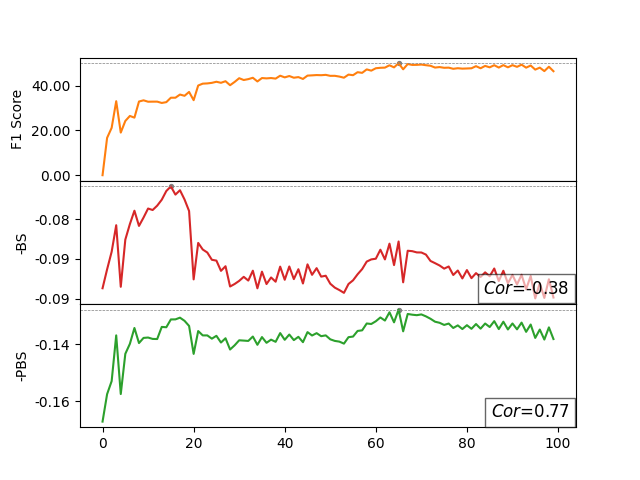}}
\caption{This figure presents an illustrative example of validation data statistics, based on epochs, for each dataset.
The orange lines show the F1 score at each epoch.
The red lines show the flipped Brier Score ($BS \times -1$).
The green lines are the flipped Penalized Brier Score ($PBS \times -1$).
Finally, the symbol $Cor$ represents the Pearson correlation between the F1 score and the corresponding scoring rule.}
\label{fig:examples}
\end{figure}

\subsection{Discussion}
Fig. \ref{fig:examples} presents validation statistics over training epochs for each dataset, including the macro-averaged F1 score, Brier Score, and proposed Penalized Brier Score. This provides an illustrative example of how the different scoring rules change during model optimization.
The Pearson correlation ($Cor$) between each scoring rule and the F1 score is also shown.
When plotting and comparing the different validation metric trends together, it's important to note that they are oriented in different directions.
Specifically, the F1 score increases positively as the model improves, while the \textit{PLL} and \textit{PBS} decrease negatively as performance increases. Therefore, to enable the results to be interpreted more easily, the \textit{PBS} values graphed have been multiplied by -1 to match the positive F1 score trend.

Importantly, the figure also marks the optimal point on each trend with a data point symbol. This optimal point corresponds to the epoch at which the metric reached its highest validation value. By pinpointing these peaks, we can see exactly where each scoring rule reached its best performance relative to the others.
Notably, the \textit{PBS} trends exhibit optimal points that are closer to the F1 score points compared to the standard Brier Score.
This observation suggests that, in these specific examples, the \textit{PBS} was more suitable for \textit{model checkpointing}.
Additionally, the slopes of the \textit{PBS} trends display greater similarity to the F1 score trends versus the standard Brier Score slopes.
Additionally, the negative correlation metric confirms the \textit{PBS} consistently maintains a stronger relationship with the F1 score than the standard Brier Score.
Therefore, given these observations, implementing \textit{early stopping} based on the \textit{PBS} rather than the Brier Score has the potential to yield models with improved test performance.

\subsection{Benchmark}
The previous discussion examining superior scoring rule trends during training provided initial insights into how the proposed metrics may help optimize model performance.
First, the optimal points of the \textit{PBS} trends were closer to those of the F1 score trends compared to the standard Brier Score. Additionally, the negative correlation between the \textit{PBS} and F1 score was consistently higher. Given this observation, it was expected that employing the \textit{PBS} for model selection could lead to improved performance. To decisively validate their effectiveness, rigorous quantitative evaluation was carried out using multiple experiments.

The first stage involved assessing the correlation between the metrics and the F1 score on validation data using k-fold cross-validation. To this end, the performance of models selected with different scoring rules using both model checkpointing (\textit{CP}) and early stopping (\textit{ES}) are evaluated.
As shown in Table \ref{tbl:correlations}, Pearson correlation coefficients were calculated between each of the scoring rules and the macro-averaged F1 scores across datasets and folds. In this table, the symbol $Cr_x$ represents the correlation between the F1 score and the scoring rule $x$. Strong positive correlation within the $-1$ to $1$ range indicates close agreement between trends.
The results demonstrate that \textit{PBS} and \textit{PLL} trends exhibited the highest degree of similarity to F1 score variations. With a correlation exceeding other scores, \textit{PBS} and \textit{PLL} can be reliably used to track changes in model performance.

Next, the ability of the proposed \textit{superior} scoring rules to select high-scoring models was evaluated through k-fold cross-validation.
Table \ref{tbl:benchmark} compares the macro-averaged F1 scores of the optimized classifier chosen by each metric on test data. These F1 scores were determined through \textit{ES} or \textit{CP}, which relied on the utilization of a superior scoring rule. The symbol $F1_x$ denotes the F1 score achieved when employing the scoring rule $x$. As depicted in the table, it is evident that the F1 score consistently performs better when the model is selected using each of the proposed \textit{superior} scoring rules. This observation emphasizes the effectiveness of the \textit{superior} scoring rules in identifying models that yield higher F1 scores. Furthermore, consistently better scores emerged when \textit{PBS} guided selection rather than \textit{PLL}.

Consequently, these findings underscore the importance of employing appropriate \textit{superior} scoring rules to circumvent shortcomings of F1 score alone in model selection, ultimately enabling improved classification capability and more trustworthy predictions. Collectively, the correlational and benchmark results provide compelling quantitative evidence that proposed penalties within strictly proper scoring rules augment their capacity to reflect true performance changes, in addition to facilitating the identification of models with stronger predictive power for new samples. Therefore, by more faithfully reflecting F1 score behavior, the proposed criteria enhance optimal model evaluation, selection, and classification for challenging spatio-temporal applications.

\begin{table}[http]
\centering
\scriptsize
\caption{The table presents a comparison based on the Pearson correlation between F1 scores and the scoring rules using validation data. The abbreviations \textit{ES} and \textit{CP} represent Early Stoping and Model Checkpointing, respectively. The symbol $Cr_x$ denotes the correlation between F1 score and the scoring rule of $x$.}
\label{tbl:correlations}
\begin{tabular}{|c|c|c|c|c|c|c|c|c|}
\hline
Data                                            & ES & CP & $Cr_{BS}$         & $Cr_{PBS}$        & $\Delta$                      & $Cr_{LL}$         & $Cr_{PLL}$        & $\Delta$                      \\ \hline
                                                   & \checkmark    &            & 0.957 ($\pm$0.02) & 0.969 ($\pm$0.01) & \cellcolor[HTML]{9AFF99}0.012 & 0.837 ($\pm$0.08) & 0.900 ($\pm$0.06) & \cellcolor[HTML]{9AFF99}0.063 \\ \cline{2-9} 
\multirow{-2}{*}{\cite{casale2012personalization}} &               & \checkmark & 0.964 ($\pm$0.01) & 0.980 ($\pm$0.01) & \cellcolor[HTML]{9AFF99}0.016 & 0.526 ($\pm$0.47) & 0.640 ($\pm$0.30) & \cellcolor[HTML]{9AFF99}0.113 \\ \hline
                                                   & \checkmark    &            & 0.963 ($\pm$0.04) & 0.983 ($\pm$0.02) & \cellcolor[HTML]{9AFF99}0.020 & 0.764 ($\pm$0.17) & 0.845 ($\pm$0.11) & \cellcolor[HTML]{9AFF99}0.081 \\ \cline{2-9} 
\multirow{-2}{*}{\cite{weiss2019smartphone}}       &               & \checkmark & 0.699 ($\pm$0.46) & 0.926 ($\pm$0.09) & \cellcolor[HTML]{9AFF99}0.228 & 0.520 ($\pm$0.53) & 0.589 ($\pm$0.45) & \cellcolor[HTML]{9AFF99}0.069 \\ \hline
                                                   & \checkmark    &            & 0.721 ($\pm$0.17) & 0.740 ($\pm$0.22) & \cellcolor[HTML]{9AFF99}0.019 & 0.731 ($\pm$0.15) & 0.745 ($\pm$0.14) & \cellcolor[HTML]{9AFF99}0.013 \\ \cline{2-9} 
\multirow{-2}{*}{\cite{torres2013sensor}}          &               & \checkmark & 0.607 ($\pm$0.49) & 0.748 ($\pm$0.45) & \cellcolor[HTML]{9AFF99}0.141 & 0.306 ($\pm$0.53) & 0.508 ($\pm$0.45) & \cellcolor[HTML]{9AFF99}0.202 \\ \hline
                                                   & \checkmark    &            & 0.690 ($\pm$0.73) & 0.748 ($\pm$0.45) & \cellcolor[HTML]{9AFF99}0.058 & 0.567 ($\pm$0.53) & 0.619 ($\pm$0.49) & \cellcolor[HTML]{9AFF99}0.052 \\ \cline{2-9} 
\multirow{-2}{*}{\cite{kaluvza2010agent}}          &               & \checkmark & 0.667 ($\pm$1.00) & 0.674 ($\pm$0.72) & \cellcolor[HTML]{9AFF99}0.007 & 0.317 ($\pm$0.65) & 0.378 ($\pm$0.65) & \cellcolor[HTML]{9AFF99}0.061 \\ \hline
                                                   & \checkmark    &            & 0.717 ($\pm$0.61) & 0.728 ($\pm$0.61) & \cellcolor[HTML]{9AFF99}0.010 & 0.275 ($\pm$0.70) & 0.292 ($\pm$0.71) & \cellcolor[HTML]{9AFF99}0.016 \\ \cline{2-9} 
\multirow{-2}{*}{\cite{scalabrini2019prediction}}  &               & \checkmark & 0.739 ($\pm$0.63) & 0.740 ($\pm$0.63) & \cellcolor[HTML]{9AFF99}0.001 & 0.384 ($\pm$0.60) & 0.425 ($\pm$0.55) & \cellcolor[HTML]{9AFF99}0.041 \\ \hline
                                                   & \checkmark    &            & 0.965 ($\pm$0.08) & 0.987 ($\pm$0.03) & \cellcolor[HTML]{9AFF99}0.022 & 0.592 ($\pm$0.16) & 0.664 ($\pm$0.13) & \cellcolor[HTML]{9AFF99}0.072 \\ \cline{2-9} 
\multirow{-2}{*}{\cite{salam2018comparison}}       &               & \checkmark & 0.995 ($\pm$0.01) & 0.997 ($\pm$0.01) & \cellcolor[HTML]{9AFF99}0.002 & 0.419 ($\pm$0.65) & 0.446 ($\pm$0.66) & \cellcolor[HTML]{9AFF99}0.026 \\ \hline
                                                   & \checkmark    &            & 0.680 ($\pm$0.31) & 0.899 ($\pm$0.05) & \cellcolor[HTML]{9AFF99}0.219 & 0.595 ($\pm$0.48) & 0.777 ($\pm$0.25) & \cellcolor[HTML]{9AFF99}0.182 \\ \cline{2-9} 
\multirow{-2}{*}{\cite{zhang2017cautionary}}       &               & \checkmark & 0.665 ($\pm$0.30) & 0.813 ($\pm$0.13) & \cellcolor[HTML]{9AFF99}0.148 & 0.378 ($\pm$0.71) & 0.802 ($\pm$0.14) & \cellcolor[HTML]{9AFF99}0.424 \\ \hline
                                                   & \checkmark    &            & 0.572 ($\pm$0.48) & 0.597 ($\pm$0.27) & \cellcolor[HTML]{9AFF99}0.025 & 0.694 ($\pm$0.23) & 0.724 ($\pm$0.21) & \cellcolor[HTML]{9AFF99}0.031 \\ \cline{2-9} 
\multirow{-2}{*}{\cite{Dua:2019}}                  &               & \checkmark & 0.444 ($\pm$0.48) & 0.704 ($\pm$0.21) & \cellcolor[HTML]{9AFF99}0.260 & 0.434 ($\pm$0.61) & 0.498 ($\pm$0.53) & \cellcolor[HTML]{9AFF99}0.064 \\ \hline
                                                   & \checkmark    &            & 0.874 ($\pm$0.14) & 0.924 ($\pm$0.07) & \cellcolor[HTML]{9AFF99}0.050 & 0.450 ($\pm$0.38) & 0.553 ($\pm$0.32) & \cellcolor[HTML]{9AFF99}0.103 \\ \cline{2-9} 
\multirow{-2}{*}{\cite{eftekhari2018hybrid}}       &               & \checkmark & 0.916 ($\pm$0.12) & 0.953 ($\pm$0.06) & \cellcolor[HTML]{9AFF99}0.037 & 0.391 ($\pm$0.52) & 0.629 ($\pm$0.32) & \cellcolor[HTML]{9AFF99}0.239 \\ \hline
\end{tabular}
\end{table}

\begin{table}[http]
\centering
\scriptsize
\caption{This table provides a comparative analysis of scoring rules using F1 scores obtained from test data, which were determined through model selection based on a scoring rule. The abbreviations \textit{ES} and \textit{CP} represent Early Stoping and Model Checkpointing, respectively. Additionally, the symbol $F1_x$ represents the F1 score achieved through the utilization of the scoring rule of $x$.}
\label{tbl:benchmark}
\begin{tabular}{|c|c|c|c|c|c|c|c|c|}
\hline
Data                                            & ES         & CP         & $F1_{BS}$         & $F1_{PBS}$        & $\Delta$                     & $F1_{LL}$         & $F1_{PLL}$        & $\Delta$                     \\ \hline
                                                   & \checkmark &            & 45.00 ($\pm$0.05) & 51.65 ($\pm$0.07) & \cellcolor[HTML]{9AFF99}6.65 & 64.76 ($\pm$0.09) & 65.85 ($\pm$0.08) & \cellcolor[HTML]{9AFF99}1.10 \\ \cline{2-9} 
\multirow{-2}{*}{\cite{casale2012personalization}} &            & \checkmark & 55.47 ($\pm$0.05) & 60.03 ($\pm$0.07) & \cellcolor[HTML]{9AFF99}4.56 & 70.00 ($\pm$0.08) & 71.82 ($\pm$0.07) & \cellcolor[HTML]{9AFF99}1.83 \\ \hline
                                                   & \checkmark &            & 53.01 ($\pm$0.06) & 57.60 ($\pm$0.04) & \cellcolor[HTML]{9AFF99}4.59 & 53.90 ($\pm$0.08) & 55.48 ($\pm$0.05) & \cellcolor[HTML]{9AFF99}1.59 \\ \cline{2-9} 
\multirow{-2}{*}{\cite{weiss2019smartphone}}       &            & \checkmark & 55.37 ($\pm$0.04) & 58.03 ($\pm$0.05) & \cellcolor[HTML]{9AFF99}2.66 & 55.74 ($\pm$0.06) & 57.63 ($\pm$0.07) & \cellcolor[HTML]{9AFF99}1.89 \\ \hline
                                                   & \checkmark &            & 28.24 ($\pm$0.07) & 32.48 ($\pm$0.09) & \cellcolor[HTML]{9AFF99}4.24 & 30.08 ($\pm$0.04) & 30.62 ($\pm$0.04) & \cellcolor[HTML]{9AFF99}0.53 \\ \cline{2-9} 
\multirow{-2}{*}{\cite{torres2013sensor}}          &            & \checkmark & 32.30 ($\pm$0.08) & 34.28 ($\pm$0.09) & \cellcolor[HTML]{9AFF99}1.99 & 31.65 ($\pm$0.08) & 31.80 ($\pm$0.05) & \cellcolor[HTML]{9AFF99}0.15 \\ \hline
                                                   & \checkmark &            & 58.27 ($\pm$0.20) & 59.14 ($\pm$0.19) & \cellcolor[HTML]{9AFF99}0.87 & 56.79 ($\pm$0.14) & 59.55 ($\pm$0.13) & \cellcolor[HTML]{9AFF99}2.76 \\ \cline{2-9} 
\multirow{-2}{*}{\cite{kaluvza2010agent}}          &            & \checkmark & 66.51 ($\pm$0.06) & 67.49 ($\pm$0.05) & \cellcolor[HTML]{9AFF99}0.97 & 68.21 ($\pm$0.11) & 69.78 ($\pm$0.08) & \cellcolor[HTML]{9AFF99}1.57 \\ \hline
                                                   & \checkmark &            & 43.66 ($\pm$0.34) & 50.81 ($\pm$0.32) & \cellcolor[HTML]{9AFF99}7.14 & 51.12 ($\pm$0.32) & 59.69 ($\pm$0.26) & \cellcolor[HTML]{9AFF99}8.57 \\ \cline{2-9} 
\multirow{-2}{*}{\cite{scalabrini2019prediction}}  &            & \checkmark & 48.00 ($\pm$0.09) & 54.83 ($\pm$0.16) & \cellcolor[HTML]{9AFF99}6.83 & 64.12 ($\pm$0.18) & 66.82 ($\pm$0.22) & \cellcolor[HTML]{9AFF99}2.71 \\ \hline
                                                   & \checkmark &            & 80.93 ($\pm$0.08) & 83.20 ($\pm$0.07) & \cellcolor[HTML]{9AFF99}2.27 & 77.87 ($\pm$0.06) & 80.13 ($\pm$0.08) & \cellcolor[HTML]{9AFF99}2.26 \\ \cline{2-9} 
\multirow{-2}{*}{\cite{salam2018comparison}}       &            & \checkmark & 82.08 ($\pm$0.08) & 83.88 ($\pm$0.07) & \cellcolor[HTML]{9AFF99}1.80 & 80.22 ($\pm$0.06) & 83.42 ($\pm$0.06) & \cellcolor[HTML]{9AFF99}3.20 \\ \hline
                                                   & \checkmark &            & 50.13 ($\pm$0.09) & 53.51 ($\pm$0.07) & \cellcolor[HTML]{9AFF99}3.38 & 55.59 ($\pm$0.08) & 56.33 ($\pm$0.08) & \cellcolor[HTML]{9AFF99}0.74 \\ \cline{2-9} 
\multirow{-2}{*}{\cite{zhang2017cautionary}}       &            & \checkmark & 52.04 ($\pm$0.11) & 53.63 ($\pm$0.10) & \cellcolor[HTML]{9AFF99}1.59 & 51.73 ($\pm$0.07) & 52.39 ($\pm$0.07) & \cellcolor[HTML]{9AFF99}0.66 \\ \hline
                                                   & \checkmark &            & 27.09 ($\pm$0.02) & 27.77 ($\pm$0.03) & \cellcolor[HTML]{9AFF99}0.69 & 23.10 ($\pm$0.04) & 23.75 ($\pm$0.02) & \cellcolor[HTML]{9AFF99}0.65 \\ \cline{2-9} 
\multirow{-2}{*}{\cite{Dua:2019}}                  &            & \checkmark & 26.51 ($\pm$0.05) & 29.37 ($\pm$0.04) & \cellcolor[HTML]{9AFF99}2.86 & 26.77 ($\pm$0.04) & 27.39 ($\pm$0.05) & \cellcolor[HTML]{9AFF99}0.62 \\ \hline
                                                   & \checkmark &            & 66.50 ($\pm$0.03) & 66.53 ($\pm$0.03) & \cellcolor[HTML]{9AFF99}0.03 & 65.39 ($\pm$0.06) & 65.81 ($\pm$0.05) & \cellcolor[HTML]{9AFF99}0.43 \\ \cline{2-9} 
\multirow{-2}{*}{\cite{eftekhari2018hybrid}}       &            & \checkmark & 67.85 ($\pm$0.03) & 68.36 ($\pm$0.02) & \cellcolor[HTML]{9AFF99}0.51 & 66.64 ($\pm$0.04) & 67.07 ($\pm$0.05) & \cellcolor[HTML]{9AFF99}0.43 \\ \hline
\end{tabular}
\end{table}

\section{Conclusion}
\label{sec:6}
This study introduced novel \textit{superior} scoring rules called Penalized Brier Score (\textit{PBS}) and Penalized Logarithmic Loss (\textit{PLL}) for evaluating probabilistic classification models. \textit{PBS} and \textit{PLL} modify the traditional Brier Score and Logarithmic Loss by integrating a penalty term for misclassified observations. As demonstrated formally, \textit{PBS} and \textit{PLL} satisfy the properties of strictly proper scoring rules while also consistently assigning superior scores to correctly classified observations. The experimental evaluation highlighted the benefits of using \textit{PBS} and \textit{PLL} for model checkpointing and early stopping. \textit{PBS} and \textit{PLL} demonstrated a higher negative correlation with the F1 score compared to traditional Brier Score and Logarithmic Loss during model training. Consequently, the proposed scoring functions were more effective in identifying optimal model checkpoints and determining early stopping points, leading to improved F1 scores. The test results substantiated that model selection based on \textit{PBS} and \textit{PLL} yielded superior F1 scores compared to traditional metrics.
In conclusion, \textit{PBS} and \textit{PLL} enable more accurate model evaluation by encapsulating both proper scoring rule principles and preferential treatment of correct classifications. The proposed metrics address a critical gap between probabilistic uncertainty assessment and deterministic accuracy maximization. By accounting for uncertainty and the value of true classifications, \textit{PBS} and \textit{PLL} can enhance model selection, checkpointing, and early stopping in classification tasks requiring reliable predictive uncertainty. Further research can explore \textit{PBS} and \textit{PLL} with different model architectures and classification problems. Also, various penalties can be investigated to obtain better performance.

%%\section{Acknowledgment}

% \clearpage
\section*{Appendix}

\section{Proof of Theorem \ref{thrm:LL_Property}}
\label{proof:LL_Property}
Let $x$ be a member of the set $\psi$ with $x_i = \alpha$, where $i$ is the index of the true class, and let $y$ be the ground-truth label vector.
Due to the single-label classification property, the Eq. (\ref{eq:logarithmic_loss}) can be expressed as:
\begin{align}
\label{eq:LL_single_label_form}
S_{LL}(x,i) = -\sum_{k=1}^{c} y_k~log(x_k)=-log(x_i)
\end{align}
Now, let $q \in \xi$ be another prediction. There are three possible cases for the true class probability $q_i$:
\begin{itemize}
\item $q_i=\alpha$.
It means that $S_{LL}(x,i) = S_{LL}(q,i)$.

\item $q_i=\alpha-\beta$, where $\beta > 0$.
Since $\alpha > \alpha - \beta$ and $-log(q_i)$ is monotonically descending and positive in $[0,1]$, it follows that $S_{LL}(q,i) > S_{LL}(x,i)$.

\item $q_i=\alpha+\beta$, where $\beta > 0$.
Since $\alpha + \beta > \alpha$ and $-log(q_i)$ is monotonically descending and positive in $[0,1]$, it follows that $S_{LL}(x,i) > S_{LL}(q,i)$.
\end{itemize}
Therefore, when $q_i \geq x_i$, the Logarithmic Loss does not assign a strictly higher score to the prediction $q$ compared to the true prediction $x$.
Hence, it cannot be considered \textit{superior}.

\section{Proof of Theorem \ref{thrm:BS_Property}}
\label{proof:BS_Property}
Let $x \in \psi$ such that $x_i=\alpha$, where $i$ is the index of the true class.
Let $y$ be the ground-truth label vector.
Due to the single-label classification property, the Logarithmic Loss can be expressed as:
Due to the single-label classification property, Eq. (\ref{eq:brier_score}) can be written as:
\begin{align}
S_{BS}(x,i) = \sum_{k=1,k\neq  i}^{c} x_k^2 + (1-x_i)^2
\end{align}
In the following, the term "hot value" is used to refer to the element of the probability vectors that correspond to the truth class, whereas the other elements are referred to as "non-hot values".
Now, let $q \in \xi$. There are three possible cases for the true class probability $q_i$:
\subsection{$q_i=\alpha$}
%Let $q \in \xi$ be an arbitrary vector.
It is possible to demonstrate that the variance of the non-hot part has the greatest impact on the score value of the Brier score function.
The Brier Score for $q$ is given by the following expression:
\begin{align}
S_{BS}(q,i) = 
\underbrace{ \sum_{k=1,k\neq  i}^{c} (q_k)^2 }_{non-hot~part} + 
\underbrace{(1-q_i)^2}_{hot~part}
\end{align}
where the index of $i$ represents the true class or $i=\arg\max\:y$.
The non-hot part can be expanded in the following manner:
\begin{align}
& \sum_{k=1,k\neq  i}^{c} (q_k)^2 
= \sum_{k=1,k\neq  i}^{c} (q_k-\tilde{q}+\tilde{q})^2 = 
\\
& \sum_{k=1,k\neq  i}^{c} \left ( (q_k-\tilde{q})^2 + 2\tilde{q}(q_k-\tilde{q}) + \tilde{q}^2 \right ) = 
\\
& \frac{c}{c} \sum_{k=1,k\neq  i}^{c} \left ( (q_k-\tilde{q})^2 + 2\tilde{q}(q_k-\tilde{q}) + \tilde{q}^2 \right )  = 
\\
& \frac{c}{c} \sum_{k=1,k\neq  i}^{c} (q_k-\tilde{q})^2 + \frac{2c\tilde{q}}{c} \sum_{k=1,k\neq  i}^{c} \left ( (q_k)-\tilde{q} \right ) + c \tilde{q}^2= 
\\
& \frac{c}{c} \sum_{k=1,k\neq  i}^{c} (q_k-\tilde{q})^2 + \frac{2c\tilde{q}}{c} \sum_{k=1,k\neq  i}^{c} (q_k)-2c\tilde{q}^2 + c\tilde{q}^2 = 
\\
& \frac{c}{c} \sum_{k=1,k\neq  i}^{c} (q_k-\tilde{q})^2 + 2c\tilde{q}^2 - 2c\tilde{q}^2 + c\tilde{q}^2 = 
\\
& \frac{c}{c} \sum_{k=1,k\neq  i}^{c} (q_k-\tilde{q})^2 + c\tilde{q}^2
\end{align}
If $\tilde{q}$ is the average of the non-hat part, then:
\begin{align} \label{BS_Non_Hot_Variance}
& \sum_{k=1,k\neq  i}^{c} (q_k)^2 
= c \left ( \underbrace{\frac{1}{c} \sum_{k=1,k\neq  i}^{c} (q_k-\tilde{q})^2}_{non-hot~part~variance}  + \tilde{q}^2 \right )
\end{align}
It is evident that:
\begin{equation}
\label{eq:eq_non_hot_mean}
\sum_{k=1,k\neq  i}^{c} q_k = 1-q_i
\Rightarrow
\tilde{q}=\frac{1-q_i}{c-1}
\end{equation}
Consequently, it can be obtained that:
\begin{equation}
\label{eq:eslvarmean}
S_{BS}(q,i) = 
c \left ( \frac{1}{c} \sum_{k=1,k\neq  i}^{c} (q_k-\tilde{q})^2 + \tilde{q}^2 \right ) + (1-q_i)^2
\end{equation}
where $i=\arg\max\:y$.
Therefore, based on Eq. (\ref{BS_Non_Hot_Variance}), the Brier Score in Eq. (\ref{eq:eslvarmean}) can be minimized by reducing the variance of the non-hot part and maximizing $q_i$.

Furthermore, $q$ may have $t$ non-hot values that exceed $\alpha$, and the sum of these $t$ non-hot values is equal to $t\alpha + \delta$. 
The following expression is utilized to formulate the sum of $c-t$ non-hot values that are less than $\alpha$:
\begin{align}
& \sum_{k=1,k\neq  i}^{c} q_k + \alpha = 1~
\\&\Rightarrow
\sum_{k=1,k\neq  i}^{c} q_k = 1 - \alpha
\\&\Rightarrow
\sum_{k=1,k\neq  i}^{c} q_k - (t\alpha+\delta) + (t\alpha+\delta) = 1 - \alpha
\\&\Rightarrow
\sum_{k=1,k\neq  i}^{c} q_k - (t\alpha+\delta) = 1 - (t+1)\alpha-\delta
\end{align}
This implies that the sum of the non-hot values is equivalent to $(t\alpha+\delta) + (1-(t+1)\alpha-\delta)$.
As $\delta$ increases and approaches $1$, the difference between $t\alpha + \delta$ and $1 - (t+1)\alpha-\delta$ also increases, leading to an increase in the variance of the non-hot part in Eq. (\ref{eq:eslvarmean}).
Therefore, it can be inferred that the variance of the non-hot part of $q$ is greater than $x$.
Let $\tilde{x}$ represent the average of the non-hat part of $x$:
\begin{align}
\left ( \frac{1}{c} \sum_{k=1,k\neq  i}^{c} (x_k-\tilde{x})^2  \right ) <
\left ( \frac{1}{c} \sum_{k=1,k\neq  i}^{c} (q_k-\tilde{q})^2  \right )
\end{align}
Now, as $q_i=x_i=\alpha$ and $\tilde{x}=\tilde{q}=\frac{1-\alpha}{c-1}$, we can proceed as follows:
\begin{align}
&
S_{BS}(x,i)-S_{BS}(q,i) = 
\\&
\left ( \sum_{k=1,k\neq  i}^{c} (x_k-\tilde{x})^2 + c\tilde{x}^2 + (1-x_i)^2 \right )-
\nonumber \\&
\left ( \sum_{k=1,k\neq  i}^{c} (q_k-\tilde{q})^2 + c\tilde{q}^2 + (1-q_i)^2 \right )=
\\&
\sum_{k=1,k\neq  i}^{c} (x_k-\tilde{x})^2 -
\sum_{k=1,k\neq  i}^{c} (q_k-\tilde{q})^2 =
\\&
c \left (\frac{1}{c}\sum_{k=1,k\neq  i}^{c} (x_k-\tilde{x})^2 -
\frac{1}{c}\sum_{k=1,k\neq  i}^{c} (q_k-\tilde{q})^2  \right ) < 0
\end{align}
Thus, the condition $S_{BS}(q,i) > S_{BS}(x,i)$ still holds even if a non-hot value is only slightly greater than $\alpha$. And $S_{BS}$ is not \textit{superior} in this case.

\begin{figure}[http]
	\centering
	\includegraphics[width=0.75\textwidth]{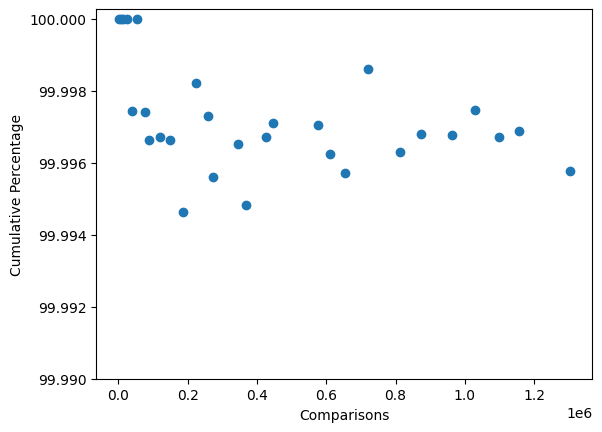}
	\caption{Cumulative percentage of instances where condition $S_{BS}(q,i) > S_{BS}(x,i)$ is met.}
	\label{fig:T4_II_MC}
\end{figure}

\subsection{$q_i=\alpha-\beta$}
To verify the validity of the condition of being \textit{superior}, a Monte Carlo simulation can be conducted, which involves comparing $S_{BS}(q,i)$ to $S_{BS}(x,i)$ for numerous randomly generated values of $x \in \psi$ and $q \in \xi$. The variables in this simulation include $x$, $q$, $\alpha$, $\beta$, and $c$, where $\alpha$ denotes the hot value of $x$, $\beta$ represents the difference between the hot value of $q$ and $\alpha$, and $c$ represents the number of classes. All of these variables are randomly generated from a normal distribution.

For each comparison, a random $x$ is selected from $\psi$ such that $x_i = \alpha$ and $x \in \mathbb{R}^c$, where $\alpha$ represents a hot value and $c$ represents the number of classes. Next, a random $q$ is chosen from $\xi$ such that $q_i = \alpha - \beta$, where $\beta$ is a random positive value. The pair $(x,q)$ is then evaluated to determine whether $S_{BS}(q,i) > S_{BS}(x,i)$ holds.

Figure \ref{fig:T4_II_MC} presents the results of the Monte Carlo simulation for varying numbers of comparisons. The figure demonstrates the convergence of the Monte Carlo method and confirms that the condition $S_{BS}(q,i) > S_{BS}(x,i)$ is satisfied in almost 99.996\% of comparisons. The figure displays the cumulative percentage of comparisons in which the condition $S_{BS}(q,i) > S_{BS}(x,i)$ holds.
When $q_i < \alpha$, the Brier Score typically assigns a higher score to the observation $q$; however, this is not always the case. Therefore, the Brier Score cannot be regarded as \textit{superior}.

\begin{figure}[http]
	\centering
	\includegraphics[width=0.75\textwidth]{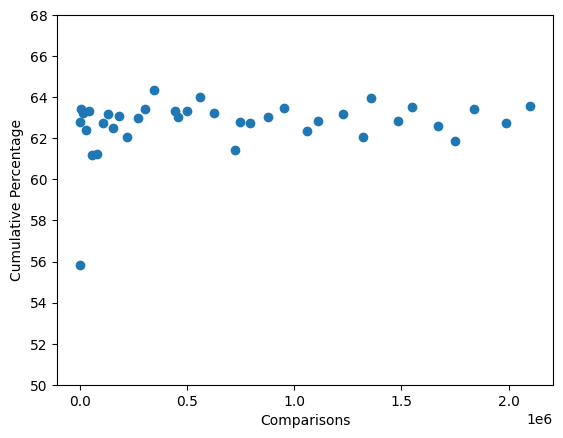}
	\caption{Cumulative percentage of instances where condition $S_{BS}(q,i) > S_{BS}(x,i)$ is met.}
	\label{fig:T4_III_MC}
\end{figure}

\subsection{$q_i=\alpha+\beta$}
To validate the condition of being \textit{superior}, it is possible to conduct another Monte Carlo simulation, which involves comparing $S_{BS}(q,i) > S_{BS}(x,i)$ for a large number of randomly generated values of $x \in \psi$ and $q \in \xi$. The simulation includes variables such as $x$, $q$, $\alpha$, $\beta$, and $c$, where $\alpha$ represents the hot value of $x$, $\beta$ represents the difference between the hot value of $q$ and $\alpha$, and $c$ represents the number of classes. All of these variables are randomly generated from a normal distribution.

In each comparison, a random value of $x$ is chosen from $\psi$ such that $x_i = \alpha$ and $x$ belongs to $\mathbb{R}^c$, where $\alpha$ denotes a hot value and $c$ represents the number of classes. Then, a random value of $q$ is selected from $\xi$ such that $q_i = \alpha + \beta$, where $\beta$ is a random positive value. The pair $(x,q)$ is evaluated to determine if $S_{BS}(q,i) > S_{BS}(x,i)$ is true.

Figure \ref{fig:T4_III_MC} depicts the outcomes of the Monte Carlo simulation for different numbers of comparisons, indicating the convergence of the Monte Carlo method. The figure illustrates the cumulative percentage of comparisons where the condition $S_{BS}(q,i) > S_{BS}(x,i)$ is satisfied. The results confirm that the condition $S_{BS}(q,i) > S_{BS}(x,i)$ holds in roughly 63\% of comparisons. 
When $q_i > \alpha$, it is common for the Brier Score to assign a higher score to the observation $x$; however, this is not always the case. Hence, the Brier Score cannot be considered \textit{superior}.

\section{Proof of Theorem \ref{thrm:MaxBS}}
\label{proof:MaxBS}
Let $x$ be an arbitrary member of the set $\psi$ and $y$ be the one-hot ground-truth label.
Let $x_i=\alpha$ where $i$ is the index of the true class.
To maximize $S_{BS}(x,i)$, the following equation is helpful:
\begin{align}
&
\sum_{k=1}^{c} x_k = 1
\Rightarrow
\sum_{k=1,k\neq i}^{c} x_k = 1 - \alpha
\\\Rightarrow &
\sum_{k=1,k\neq i}^{c} x_k^2 \leq (1-\alpha)^2
\end{align}
Since $(1-\alpha)^2$ is the primary component of the Brier Score, the maximization of $(1-\alpha)^2$ (or the minimization of $\alpha$) is essential to obtain the highest possible value of $S_{BS}(x,i)$. The minimum value of $\alpha$ is $\frac{1}{c}+\epsilon$, as any value of $\alpha$ below this threshold would make $x$ unsuitable for inclusion in the set $\psi$. As a result, $\alpha$ is equivalent to $\frac{1}{c}+\epsilon$.
On the other hand, if $x_i=\frac{1}{c}+\epsilon$, the remaining elements of the vector $x$ cannot exceed $\frac{1}{c}+\epsilon$. To simplify the analysis, assuming $\epsilon=0$, the elements of $x$ would be $x_k=\frac{1}{c},~\forall k \in [1,\cdots,c]$. Therefore:
\begin{align}
\max S_{BS}(x,i) &= 
\sum_{k=1,k\neq i}^{c}x_k^2 + (1-x_i)^2
\\&= 
\sum_{k=1,k\neq i}^{c}\frac{1}{c^2} + (1-\frac{1}{c})^2
\\&= 
(c-1) \frac{1}{c^2} + (1-\frac{1}{c})^2
\\&= 
\frac{1}{c} - \frac{1}{c^2}+1-\frac{2}{c}+\frac{1}{c^2}
\\&= 
1 - \frac{1}{c}
= 
\frac{c-1}{c}
\end{align}

\section{Proof of Theorem \ref{thrm:MaxLL}}
\label{proof:MaxLL}
Let $x$ be an arbitrary member of the set $\psi$, and $y$ be the one-hot ground-truth label vector such that $x_i=\alpha$, where $i$ is the index of the true class.
According to Eq. (\ref{eq:LL_single_label_form}) and since $-log$ is a decreasing and positive function in the range $[0,1]$, minimizing $\alpha$ is necessary to maximize $S_{LL}(x,i)$.
The minimum value of $\alpha$ is $\frac{1}{c}+\epsilon$, as any value of $\alpha$ below this threshold would render $x$ unsuitable for inclusion in the set $\psi$. To simplify the analysis, we assume $\epsilon=0$, which yields:
\begin{align}
\max S_{LL}(x,i) = -log(\frac{1}{c})
\end{align}

\section{Proof of Theorem \ref{thrm:PBSPLLScoringRules}}
\label{proof:PBSPLLScoringRules}
As $S_{BS}$ and $S_{LL}$ are strictly proper, so:
\begin{align}
S_{BS}(P,Q) >  S_{BS}(Q,Q)
\\
S_{LL}(P,Q) >  S_{LL}(Q,Q)
\end{align}
for $Q \neq P$.
Furthermore, it is clear that:
\begin{align}
S_{BS}(Q,Q) = S_{PBS}(Q,Q)
\\
S_{LL}(Q,Q) = S_{PLL}(Q,Q)
\end{align}
and also: 
\begin{align}
S_{PBS}(P,Q) \geq S_{BS}(P,Q)
\\
S_{PLL}(P,Q) \geq S_{LL}(P,Q)
\end{align}
Therefore:
\begin{align}
S_{PBS}(P,Q) \geq S_{BS}(P,Q) > S_{BS}(P,Q)
\\
S_{PLL}(P,Q) \geq S_{BS}(P,Q) > S_{LL}(P,Q)
\end{align}
As a result, $S_{PBS}$ and $S_{PLL}$ are \textit{strictly proper}.

\bibliography{cas-refs}

\end{document}